\documentclass{article}

\usepackage{microtype}
\usepackage{graphicx}
\usepackage{booktabs} 

\usepackage{amsthm}
\newtheorem{theorem}{Theorem}
\newtheorem{lemma}{Lemma}

\usepackage{bm} 
\usepackage{float} 
\usepackage{amssymb}
\usepackage{mathtools}
\usepackage{cancel}
\usepackage{caption}
\usepackage{subcaption}

\usepackage{hyperref}

\usepackage{icml2021}

\def\phiprior{\phi_K^{\text{prior}}}
\def\phipost{\hat{\phi}_K}

\DeclarePairedDelimiterX{\infdivx}[2]{(}{)}{%
	#1\;\delimsize\|\;#2%
}
\newcommand{\KL}{\mathbb{K}\mathbb{L}\infdivx}

\icmltitlerunning{MFVI BNNs ignore the data}

\begin{document}

\twocolumn[
\icmltitle{Wide Mean-Field Variational Bayesian Neural Networks Ignore the Data}



\icmlsetsymbol{equal}{*}

\begin{icmlauthorlist}
\icmlauthor{Beau Coker}{bio}
\icmlauthor{Weiwei Pan}{seas}
\icmlauthor{Finale Doshi-Velez}{seas}
\end{icmlauthorlist}

\icmlaffiliation{seas}{John A. Paulson School of Engineering and Applied Sciences, Harvard University, Cambridge, MA, USA}
\icmlaffiliation{bio}{Department of Biostatistics, Harvard T.H. Chan School of Public Health, Boston, MA, USA}

\icmlcorrespondingauthor{Beau Coker}{beaucoker@g.harvard.edu}

\icmlkeywords{Machine Learning, ICML}

\vskip 0.3in
]



\printAffiliationsAndNotice{\icmlEqualContribution} 

\begin{abstract}
Variational inference enables approximate posterior inference of the highly over-parameterized neural networks that are popular in modern machine learning. Unfortunately, such posteriors are known to exhibit various pathological behaviors. We prove that as the number of hidden units in a single-layer Bayesian neural network tends to infinity, the function-space posterior mean under mean-field variational inference actually converges to zero, completely ignoring the data. This is in contrast to the true posterior, which converges to a Gaussian process. Our work provides insight into the over-regularization of the KL divergence in variational inference. 
\end{abstract}

\section{Introduction}

Bayesian neural networks (BNNs) provide principled notions of uncertainty in deep learning, but they have not been widely adopted in practice as many properties of this model and its inference remain poorly understood, particularly in the overparameterized, nearly nonparametric regime where modern deep learning take often places. One tool for understanding the behavior of BNNs with a large number of parameters is to take the limit as the number of hidden units (i.e., \emph{width}) goes to infinity. In this case, the prior predictive distribution of a single-layer BNN converges in distribution to the \textit{NNGP}, a Gaussian process with the \emph{neural network kernel} that depends on the prior and architecture of the network \citep{neal_1996}. Extensions of this result exist also for deep networks  \citep{matthews_2018} and for BNN posteriors  \citep{hron_2020}. 

In contrast, asymptotic properties of popular variational approximations of BNN posteriors have not been extensively studied.  In the finite width regime, variational posterior approximations are known to have deficiencies.  For example, the commonly used mean-field variational posterior, which ignores correlations between parameters, exhibits various pathologies including poor uncertainty estimates between data rich regions \citep{foong_2019_2}. Moreover, recent works have noted the tendency of mean-field variational BNN posteriors to underfit even with large network architectures \citep{dusenberry2020efficient}. It is therefore natural to ask if the deficiencies in mean-field variational approximations of finite width BNN posteriors persist in approximations of infinite width BNN posteriors.

In this paper, we show that the answer, unfortunately, is \emph{yes}. For single-layer Bayesian neural networks used for univariate regression, we prove a surprising result: the posterior predictive mean under mean-field variational inference converges to zero (assuming the observed outcomes are centered) as the number of hidden units tends to infinity. That is, the mean-field variational posterior predictive mean completely ignores the data, unlike the true posterior predictive of an infinite width BNN. Furthermore, we provide insight on the cause of this failure --- we show that this results from the over-regularizing effect of the KL divergence term (forcing the posterior to match the zero-mean prior) in the variational inference objective. 


\section{Background}

\paragraph{Bayesian neural networks (BNNs)}

A single-layer feedforward neural network with $K$ hidden units used for univariate regression is given by:
\begin{equation}
	f(x, \theta) = \sum_{k=1}^K w^{(2)}_k \psi(w^{(1)T}_k x + b^{(1)}_k )
	\label{eq:bnn}
\end{equation}
where $\psi$ is a nonlinear activation function, $x\in \mathbb{R}^D$ is an input, $w^{(1)}_k \in \mathbb{R}^D$ and $b^{(1)}_k \in \mathbb{R}$ are input-layer weight and bias parameters, respectively, and $w^{(2)}_k \in \mathbb{R}$ is an output-layer weight parameter. For simplicity, we assume the mean has been subtracted from the observed outcomes so that we can omit an output-layer bias parameter. We let $\theta \in \mathbb{R}^{K(D+2)}$ denote the collection of all model parameters. Given independent and identically distributed observations $\mathcal{D} = \{(x_n, y_n)\}_{n=1}^N$, we assume a Gaussian likelihood function $\mathcal{L}(\theta) = \prod_n \mathcal{N}(y_n \mid f(x_n, \theta), \sigma^2_\text{noise})$ and infer the parameters $\theta$ by maximizing the likelihood.

A \textit{Bayesian neural network} (BNN) places a prior distribution $p(\theta)$ on the parameters, typically a factorized Gaussian:
\begin{align}
	w^{(1)}_k&\sim \mathcal{N}(0, \sigma_{w^{(1)}}^2 I_D)
	\label{eq:prior1}
	\\
	b^{(1)}_k &\sim \mathcal{N}(0, \sigma_{b^{(1)}}^2)
	\label{eq:prior2}
	\\
	w^{(2)}_k &\sim \mathcal{N}(0, \sigma_{w^{(2)}}^2),
	\label{eq:prior3}
\end{align}
where $I_D$ is a $D\times D$ identity matrix and $k\in \{1,\dotsc,K\}$.

\paragraph{Convergence to Gaussian processes} 
Under this prior, the covariance between the function $f$ evaluated at two inputs $x$ and $x'$ is given by:
\begin{align*}
	K  \sigma_{w^{(2)}}^2 \mathbb{E}[\psi(w^{(1)T}_k x + b^{(1)}_k) \psi(w^{(1)T}_k x' + b^{(1)}_k)]
\end{align*}
Notice by scaling the output-layer prior variance parameter $\sigma_{w^{(2)}}^2$ inversely with the network width $K$ (i.e., setting $\sigma_{w^{(2)}}^2 = \tilde{\sigma}_{w^{(2)}}^2/K$ for some $\tilde{\sigma}_{w^{(2)}}^2$), the prior covariance is the same for any width. Then, as the number of hidden units tends to infinity, application of the central limit theorem reveals that for any set of inputs $x$, the prior predictive distribution over the function output $f(x)$ approaches a multivariate Gaussian with the covariance given above. This  the definition of a \textit{Gaussian process} (GP). In this case, it is called the neural network Gaussian process (NNGP), since the covariance function of the multivariate Gaussian distribution over $f(x)$ is induced by the neural network.

Note that while scaling the output-layer prior variance parameter by $1/K$ is useful for analyzing the theoretical properties of a BNN, it is an important practical consideration, too. Otherwise the prior predictive variance could become too large as the width increases. 

\paragraph{Variational inference}
Unfortunately, the posterior distribution of a finite-width BNN is not available in closed-form and Markov chain Monte Carlo (MCMC) methods are too slow for all but the smallest networks. Instead, it is common to find the closest distribution $q_\phi$ in KL divergence to the posterior by maximizing a lower bound on the marginal likelihood called the evidence lower bound (ELBO):
\begin{equation}
	\label{eq:elbo}
	\text{ELBO}(\phi) = \mathbb{E}_{\theta\sim q_\phi}[\log \mathcal{L}(\theta)] - \mathbb{K}\mathbb{L}[q_\phi || p(\theta)].
\end{equation}
The first term in the ELBO is the expected log likelihood, which measures how well the model fits the data, and the second term is the Kullback-Leibler (KL) divergence regularization, which measure how close $q_\phi$ is to the prior $p(\theta)$. 

The \textit{variational} distribution $q_\phi$ is parameterized by a set of variational parameters $\phi$. A common choice for $q_\phi$ is a product of independent (i.e., \textit{mean-field}) Gaussian distributions, one distribution for each parameter $\theta_i$ in the model:
\begin{equation}
	q_\phi(\theta) = \prod_{i=1}^{|\theta|} \mathcal{N}(\theta_i \mid \mu_i, \sigma^2_i).
	\label{eq:q}
\end{equation}
Here, $\phi = \{(\mu_i, \sigma^2_i)\}$ are the variational parameters. We call variational inference using Equation \ref{eq:q} \textit{mean-field variational inference} (MFVI).

Since both the prior and variational distribution are Gaussian, the KL divergence can be calculated in closed-form. For example, the KL divergence between the variational distribution of all $K$ output-layer weights $w^{(2)}$ and a $\mathcal{N}(0, \frac{1}{K}I_K )$ prior is 
\begin{align}
	\frac{1}{2} \sum_{k=1}^{K} \left[K \mu_k^2 + K\sigma_k^ 2 -1 - \log K \sigma_k^2 \right],
	\label{eq:kl}
\end{align}
where $\{\mu_k\}$ and $\{\sigma^2_k\}$ are the variational parameters. Notice Equation \ref{eq:kl} acts like L2 regularization of the mean parameters, which will be key to our proof of Theorem \ref{th:1}.


\begin{figure}[H]
	\centering
	\includegraphics[width=.47\textwidth]{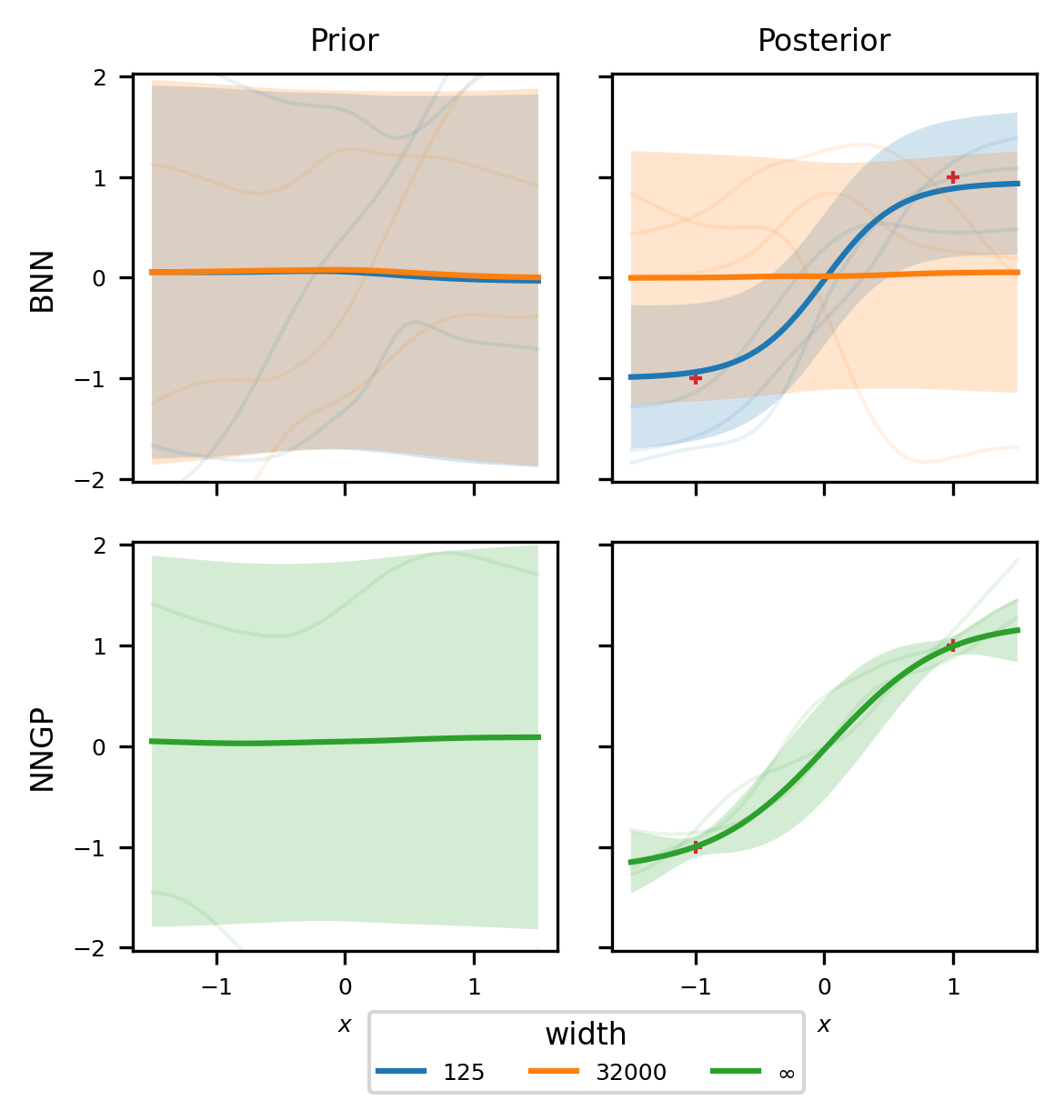}
	\vspace{-.3cm}
	\caption{Prior and posterior predictive distributions for mean-field variational BNNs of different widths compared to the NNGP, to which the true posterior of the BNN converges. For a large width, the mean-field variational BNN ignores the data, unlike the NNGP. The shaded regions constitute $\pm 1$ standard deviation around the means (solid lines). All estimates are based on 1000 function samples (a few of which are drawn faintly).}
	\label{fig:posteriors}
\end{figure}

\section{Main result} \label{sec:main}

Figure \ref{fig:posteriors} illustrates our main observation. The left column compares the prior predictive distributions of BNNs of various widths (top panel) to the Gaussian process prior to which these BNNs converge (i.e., the NNGP, bottom panel). As expected from the scaling in the prior, the prior predictive mean and variance is the same between all finite and infinite width models. However, the right column reveals significant differences in the posterior predictive distributions. While the NNGP posterior exhibits a posterior mean that models the data and a posterior variance that expands outside of the data, the MFVI BNN posterior predictives completely ignore the data as the width increases. We use the implementation provided by \cite{novak_2019} to compute the NNGP.

Like other BNN pathologies observed in practice, it is not immediately obvious whether the behavior is due to the variational objective itself or poor optimization of the ELBO. However, in Theorem \ref{th:1}, we show that the underfitting observed in wide BNNs is due to the choice of the variational objective itself (i.e., the combination of the choice of the variational family, the prior and the divergence measure in variational inference). Specifically, we prove that the MFVI posterior mean converges to zero as the width approaches infinity. Although we assume an error function activation function, we empirically demonstrate that the same conclusions likely hold other activations (we provide results for ReLU and tanh activations in Section \ref{sec:experiments}).

\begin{theorem}
	\label{th:1}
	Consider a mean-field variational BNN of width $K$ as described by equations \ref{eq:bnn} and \ref{eq:prior1}-\ref{eq:prior3} and assume an error function activation $\phi=\mathrm{erf}(z):=\int_{0}^z 2/\sqrt{\pi} \exp(-t^2)~dt$. Suppose $\hat{\phi}_K$ maximizes the ELBO given in Equation \ref{eq:elbo}. Then for any dataset $\mathcal{D}$ and input $x^*$, the variational posterior predictive mean $\mathbb{E}_{\theta \sim q_{\hat{\phi}_K}}[f(x^*,\theta)]$ converges to zero as the number of hidden units $K\to\infty$.
\end{theorem}

To understand why Theorem \ref{th:1} holds, think of the single-layer neural network in Equation \ref{eq:bnn} as a linear model using $K$ nonlinear basis functions $\psi(w^{(1)T}_k x + b^{(1)}_k )$. Upon observing data, as the number of basis functions $K$ increases, the prior encourages each of them to be weighted less in the posterior, because of the $1/K$ scaling of the prior variance of the output-layer weights $w^{(2)}$. This tradeoff is what allows the prior predictive variance to converge as $K\to\infty$. 

Unfortunately, because of the KL regularization of input-layer parameters $w^{(1)}$ and $b^{(1)}$ in the ELBO, basis functions cannot freely be added to the model without further penalty. Recall that the number of observations is fixed, so there is little improvement to the expected log likelihood term in the ELBO once the neural network has enough hidden units to fit the data. This limits the degree to which the variational distribution over the input-layer parameters can differ from the prior and, in the limit of $K$, prevents the model from fitting the data. 

In contrast, a random features model (e.g., Random Fourier Features \cite{rahimi_2007}), where the input-layer parameters are drawn from a fixed distribution at initialization, has no such issue. Additional basis functions are drawn from the same distribution for all $K$ and can be added to the model without penalty. 

The formal proof can be found in the Appendix, but we give a brief sketch here. The proof of Theorem \ref{th:1} follows in roughly three steps. 
\begin{itemize}
\item \textbf{Step 1: evaluate the ELBO at the prior}
Notice the ELBO evaluated at the prior does not depend on the width $K$. This provides a lower bound on the ELBO that holds for networks of any width.

\item \textbf{Step 2: bound the KL regularization}
Because the ELBO cannot be lower than the bound, the KL divergence cannot be higher than the bound. In particular, this constrains the L2 norm of the optimal variational mean parameters (see Equation \ref{eq:kl}).

\item \textbf{Step 3: bound the posterior mean} 
Using the constraint on variational mean parameters, application a few basic inequalities to the posterior mean reveals convergence to zero. 
\end{itemize}

While Theorem \ref{th:1} shows that the variational posterior predictive mean converges to the prior predictive mean, we provide empirical evidence that the variational posterior predictive variance converges also to the prior predictive variance. In particular, Figure \ref{fig:varational_params} in the supplementary material shows that the variational posterior variances approaches the prior variances and Figure \ref{fig:convergence} shows that the variational predictive variance becomes more similar to the prior predictive variance as we increase network width. We thus conjecture that a stronger statement is true: that the mean-field variational posterior distribution converges to the prior. A proof of this is current work. Finally, although Theorem \ref{th:1} applies to single hidden layer networks, understanding the asymptotic properties of deep networks is also work in progress.

\begin{figure}[H]
	\centering
	\includegraphics[width=.48\textwidth]{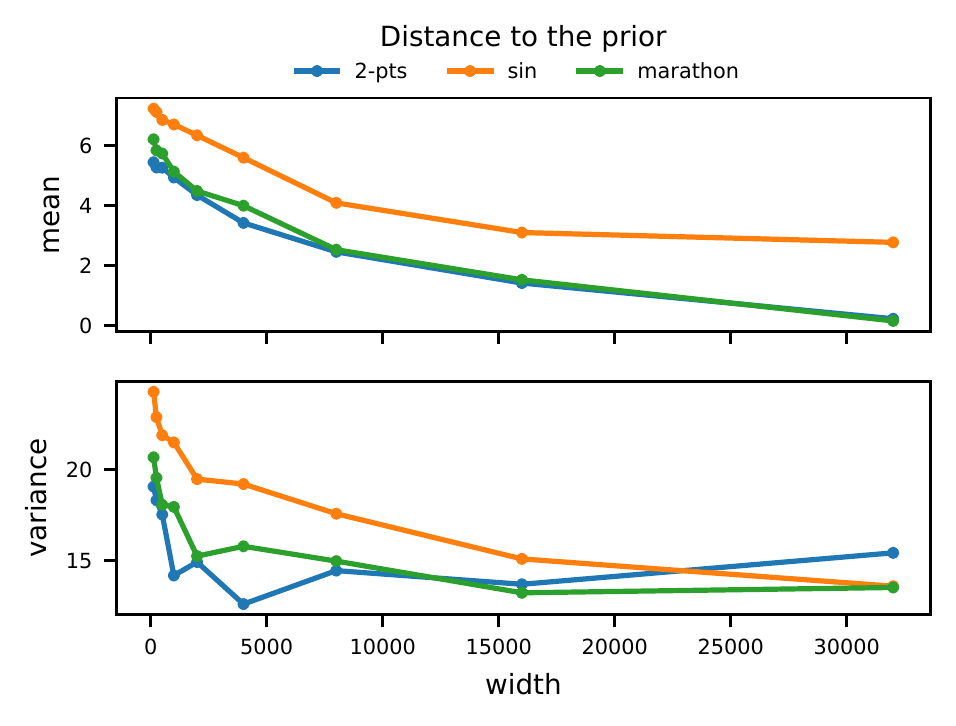}
	\vspace{-.8cm}
	\caption{Across datasets, the mean-field variational posterior of a BNN gets closer to the prior as the width increases. \textit{Top:} Euclidean distance between the prior and posterior predictive mean. \textit{Bottom:} Euclidean distance between the prior and posterior predictive variance. We use 50 test points spaced uniformly over a grid centered roughly around the training data.}
	\label{fig:convergence}
\end{figure}

\begin{figure}[H]
	\centering
	\includegraphics[width=.45\textwidth]{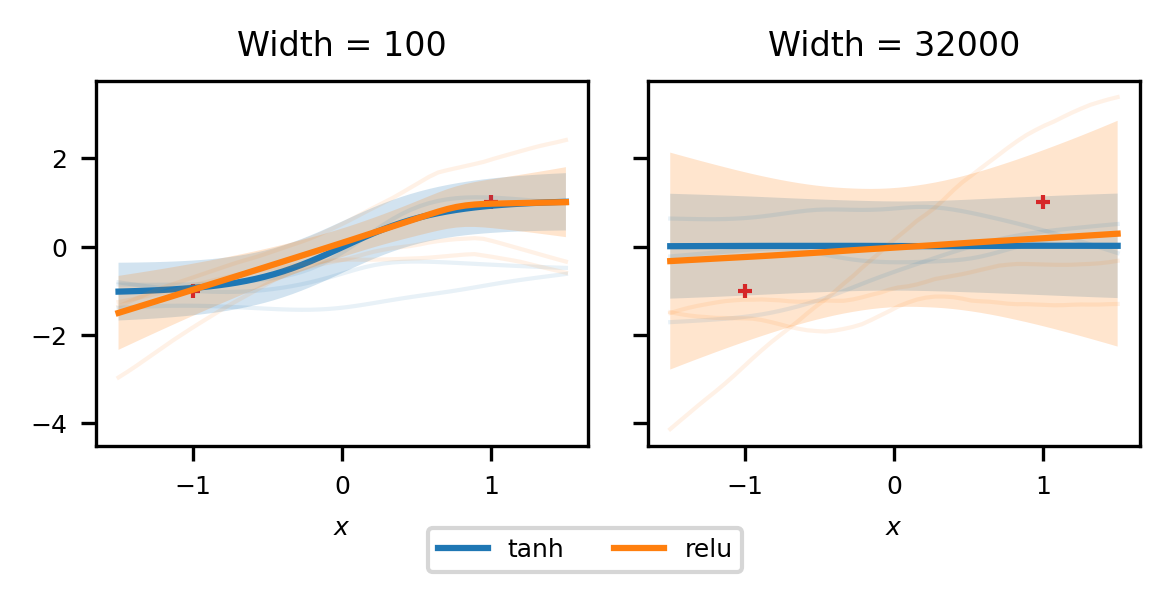}
	\vspace{-0.5cm}
	\caption{Comparison of the posterior predictives of MFVI BNNs using tanh and ReLU activation functions for small (left panel) and large (right panel) widths. Regardless of the activation, wide MFVI BNNs ignore the data.}
	\label{fig:relu}
\end{figure}

\section{Experiments}
\label{sec:experiments}

We begin by analyzing the rate of convergence of the posterior predictive mean of the mean-field variational BNN to zero for different datasets. The datasets are the ``2-points'' dataset shown in Figures \ref{fig:posteriors} and \ref{fig:relu}, a synthetic dataset of noisy observations of a sine wave ($N=20$), and a real dataset of winning Olympic marathon paces over the period 1896 to 2012 ($N=27$).\footnote{Available in the open-source \texttt{pods} Python package \url{https://github.com/sods/ods}.} We $z$-score standardize the inputs and outputs of all datasets. 

For each dataset and widths ranging from $K=125$ to $K=32,000$, the top and bottom panels, respectively, of Figure \ref{fig:convergence} show the Euclidean distance of the posterior predictive mean  and variance under mean-field variational inference to the prior predictive mean and variance. The top panel illustrates the conclusions of Theorem \ref{th:1}: as the width increases, the posterior predictive mean converges to zero. In the bottom panel panel, we see the posterior predictive variance getting closer to the prior, though it is still far away. We will investigate this behavior in future work.

Next we analyze the impact of different activation functions. Although Theorem \ref{th:1} assumes an erf activation function, Figure \ref{fig:relu} provides empirical evidence that the results of Theorem \ref{th:1} hold for tanh and ReLU activations as well.

\subsection{Implementation details}
We use a prior variance of 2 for all parameters (i.e., $\sigma_{w^{(1)}}^2=2$,  $\sigma_{b^{(1)}}^2=2$, and $\tilde{\sigma}_{w^{(2)}}^2=2$) and a prior observational noise $\sigma^2_{\text{noise}}=.01$. However, we implement the impact of the prior distribution by scaling the parameters by their prior variance in the forward pass and then using a $\mathcal{N}(0,1)$ distribution as the prior to evaluate the KL divergence term in the ELBO. That is, for any parameter $\theta_0 \in \theta$, where $p(\theta_0) = \mathcal{N}(0, \sigma_{\theta_0}^2)$ is its prior distribution, we replace $\theta_0$ with $\sigma_{\theta_0} \theta_0 $ in the evaluation of $f(x, \theta)$ and  use a $\mathcal{N}(0,1)$ as the prior. In particular, this means that instead of scaling the output-layer prior variance $\sigma_{w^{(2)}}^2$ by $1/K$, we scale the function output by $1/K$ and keep $\sigma_{w^{(2)}}^2$  unscaled. This ``neural tangent kernel'' scaling yields the same prior predictive distribution while enabling a constant learning rate to be used for training networks of different widths \cite{lee_2019}.  

We initialize the variational mean and variance parameters from a normal-inverse-gamma family. Specifically, if $q(\theta_0) = \mathcal{N}(\mu, \sigma^2)$ is the variational distribution corresponding to the parameter $\theta_0$, we randomly initialize $\mu\sim\mathcal{N}(0,1)$ and $\sigma^2 \sim \mathcal{I}\mathcal{G}(\nu+1, \nu)$. It follows from the laws of total expectation and variance that $\mathbb{E}[\theta_0] = 0$ and $\mathbb{V}[\theta_0] = 2$. The hyperparameter $\nu$ controls the concentration of $\sigma^2$ around its initial mean of one (i.e., $\mathbb{E}[\sigma^2]=1$ and $\mathbb{V}[\sigma]=1/(\nu-1)$). We set $\nu$ to the width of the network. We have experimented with other initializations, including very small variances and  mean parameters pretrained to maximize the log likelihood, with little impact on the overall results.

\section{Discussion}


\citet{trippe_2018} discusses \textit{over-pruning}, which is the phenomenon whereby many of the variational distributions over the output-layer weights concentrate around zero. This is undesireable behavior because the amount of over-pruning increases with the degree of over-parameterization and because over-pruning degrades performance --- simpler models that do not permit pruning often perform better. As in our work, the explanation for over-pruning centers around the tension between the likelihood term and the KL divergence term in the ELBO. To reduce the KL divergence, hidden units can be pruned from the model. Relatedly, we prove that as the number hidden units tends to infinity, the KL divergence over-regulates the model by pulling each of the output-layer weights towards zero while limiting the overall probability mass assigned to the input-layer weights (see earlier discussion in Section \ref{sec:main} regarding the intuition of Theorem 1). 

Our work also provides theoretical insight into \textit{cold posteriors}, which is the empirical phenomenon that downweighting the importance of the KL divergence in the ELBO yields better model performance \cite{wenzel_2020}. It is possible this practice serves to undo the over-regularization of the KL divergence that we investigate. 

A common theme in both of these pathologies is that the effect becomes more prominent as the width of the network increases. Yet, the phenomenon of \textit{double descent} shows that it is in this over-parameterized regime where neural networks perform best \cite{belkin_2019}. Therefore, it is critical to understand the properties of wide variational BNNs --- which we prove are considerably different from the true posterior in the mean-field case --- if they are to be adopted in practice. 

%


\bibliography{udl}
\bibliographystyle{icml2021}

%
%
%

\newpage
\onecolumn
\appendix
\section*{Supplementary material}
\section{Proofs}

\subsection{Lemmas}

\begin{lemma}
	\label{lemma1}
	Let $z\sim \mathcal{N}(\mu, \sigma^2)$. Then
	\begin{equation}
		\mathbb{E}[\mathrm{erf}(z)] = \mathrm{erf}\left(\frac{\mu}{\sqrt{1+2\sigma^2}} \right).
	\end{equation}
\end{lemma}
\begin{proof}
	First we claim 
	\begin{align}
		\mathbb{E}[\Phi(z)] = \Phi\left(\frac{\mu}{\sqrt{1+\sigma^2}} \right),
		\label{eq:lemma1_0}
	\end{align}
	where $\Phi(z) := \int_{-\infty}^z \mathcal{N}(t \mid 0, 1)~dt$ is the cumulative distribution function of a standard normal distribution. To see this, let $x \sim \mathcal{N}(0,1)$ and notice $P(x \le z \mid z=t) = P(x \le t) = \Phi(t)$. By the law of total probability:
	\begin{align}
		P(x \le z) &= \int P(x\le z \mid z=t) p(z=t) ~dt
		\\
		&= \int \Phi(t) ~\mathcal{N}(z=t \mid \mu, \sigma^2) ~dt
		\\
		&= \mathbb{E}[\Phi(z)].
		\label{eq:lemma1_1}
	\end{align}
	Now, since $x$ and $z$ are independent, notice $x - z \sim \mathcal{N}(-\mu, 1+\sigma^2)$. Therefore, 
	\begin{align}
		P(x \le z) &= P(x - z \le 0)  = \Phi\left(\frac{\mu}{\sqrt{1+\sigma^2}} \right).
		\label{eq:lemma1_2}
	\end{align}
	Equation \ref{eq:lemma1_0} follows from Equations \ref{eq:lemma1_1} and \ref{eq:lemma1_2}.
 
 	Noting $\text{erf}(z) = 2 \Phi(\sqrt{2} z) -1$ and applying Equation \ref{eq:lemma1_0}  we have the desired result:
	\begin{align}
		\mathbb{E}[\mathrm{erf}(z)] &= \mathbb{E}[2 \Phi(\sqrt{2} z) -1]
		\\
		&= 2\mathbb{E}[\Phi(\sqrt{2} z)] - 1
		\\
		&= 2 \Phi\left(\frac{\sqrt{2} \mu}{\sqrt{1+2\sigma^2}} \right) - 1
		\\
		&= 2 \left( \frac{1}{2} \mathrm{erf}\left( \frac{\mu}{\sqrt{1+2\sigma^2}}\right) + \frac{1}{2}\right) - 1
		\\
		&= \mathrm{erf}\left( \frac{ \mu}{\sqrt{1+2\sigma^2}}\right).
	\end{align}
	
\end{proof}

\begin{lemma}
	\label{lemma2}
	For all $z\in\mathbb{R}$,
	\begin{equation}
		\mathrm{erf}(z)^2 \le 1-\exp(-\frac{4}{\pi} z^2).
	\end{equation}
\end{lemma}
\begin{proof}
	Define
	\begin{equation}
	G(z) = \int_{0}^z (2\pi)^{-1/2} \exp\left(-\frac{1}{2} t^2\right)~dt.
	\end{equation}
	For any $z\ge 0$, \citet{polya_1949} proved the following inequality (see Equation 1.5 in the referenced paper):
	\begin{equation}
		G(z) \le \frac{1}{2}\left(1 - \exp\left(-\frac{2}{\pi} z^2 \right) \right)^{1/2}.
		\label{eq:lemma2_polya}
	\end{equation}
	Since $\text{erf}(z) := \int_0^z 2/\sqrt{\pi} \exp(-t^2)~dt$, the change of variables $s=t/\sqrt{2}$ shows $G(z) = \text{erf}(z/\sqrt{2}) / 2$. Equation \ref{eq:lemma2_polya} is therefore equivalent to
	\begin{align}
		\frac{1}{2} \text{erf}(z/\sqrt{2}) &\le \frac{1}{2}\left(1 - \exp\left(-\frac{2}{\pi} z^2 \right) \right)^{1/2}
		\\
		\text{erf}(z/\sqrt{2})^2 &\le 1 - \exp\left(-\frac{2}{\pi} z^2 \right)
		\\
		\text{erf}(z)^2 &\le 1 - \exp\left(-\frac{4}{\pi} z^2 \right),
	\end{align}
	where the final inequality comes from evaluating at $z \leftarrow \sqrt{2} z$. Note that so far we have assumed $z\ge0$, but notice each side of the final inequality is the same for $z<0$ (i.e., 
	$\text{erf}(-z)^2 = \text{erf}(z)^2$ and $1 - \exp\left(-\frac{4}{\pi} (-z)^2 \right) = 1 - \exp\left(-\frac{4}{\pi} z^2 \right)$), so the final inequality holds for all $z\in\mathbb{R}$. 
\end{proof}

\begin{lemma}
	\label{lemma3}
	Assume $\sum_{k=1}^K \mu_k^2 \le C_0$, where $C_0\in\mathbb{R}$ and 
	$\mu_k \in \mathbb{R}$ for all $k\in\{1,\dotsc,K\}$. Then, for any constant $C_1>0$,
	\begin{equation}
	\sum_{k=1}^K \exp\left(-C_1 \mu_k^2 \right) \ge K \exp\left(-\frac{1}{K}C_1 C_0 \right).
	\end{equation}
\end{lemma}
\begin{proof}
	By assumption,
	\begin{align}
		\sum_{k=1}^K \mu_k^2 &\le C_0 
		\\
		-\frac{1}{K}\sum_{k=1}^K C_1 \mu_k^2 & \ge -\frac{1}{K}C_1 C_0
		\\
		\exp\left( -\frac{1}{K}\sum_{k=1}^K C_1 \mu_k^2 \right) &\ge \exp\left(-\frac{1}{K}C_1 C_0 \right).
	\end{align}
By Jensen's inequality using the convex function $\exp(\cdot)$,
\begin{align}
	\frac{1}{K}\sum_{k=1}^K \exp\left( -C_1 \mu_k^2\right) &\ge \exp\left( \frac{1}{K}\sum_{k=1}^K (-C_1 \mu_k^2) \right)
	\\
	&= \exp\left(- \frac{1}{K}\sum_{k=1}^K C_1 \mu_k^2 \right)
	\\
	&\ge \exp\left( -\frac{1}{K}C_1 C_0 \right).
\end{align}
Multiplying each side by $K$ gives the desired result. 
\end{proof}

\begin{lemma}
	\label{lemma4}
	Let $\sum_{k=1}^K a_{km}^2 \le C_m$ for all $m\in\{1,\dotsc,M\}$, where $C_m\in\mathbb{R}$ for all $m\in\{1,\dotsc,M\}$ and  $a_{km}\in\mathbb{R}$ for all $m\in\{1,\dotsc,M\}$ and all $k\in\{1,\dotsc,K\}$. Then
	\begin{equation}
		\sum_{k=1}^K \left( \sum_{m=1}^M a_{km} \right)^2 \le M \sum_{m=1}^M C_m.
	\end{equation}
\end{lemma}
\begin{proof}
	Let $k\in \{1,\dotsc,K\}$. By the Cauchy–Schwarz inequality:
	\begin{align}
		\left( \sum_{m=1}^M a_{km} \right)^2  &= \left( \sum_{m=1}^M 1 \cdot a_{km} \right)^2
		\\
		&\le \left(\sum_{m=1}^M 1^2 \right) \left(\sum_{m=1}^M a_{km}^2 \right)
		\\
		&= M \sum_{m=1}^M a_{km}^2.
	\end{align}
	Therefore,
	\begin{align}
		\sum_{k=1}^K \left( \sum_{m=1}^M a_{km} \right)^2 
		&\le \sum_{k=1}^K \left( M \sum_{m=1}^M a_{km}^2 \right)
		\\
		&= M \sum_{m=1}^M \sum_{k=1}^K a_{km}^2 
		\\
		&\le M \sum_{m=1}^M C_m
	\end{align}
\end{proof}

\subsection{Proof of Theorem \ref{th:1}}

\begin{proof}
	For simplicity, assume the prior variance parameters ($\sigma_{w^{(1)}}^2$,  $\sigma_{b^{(1)}}^2$, and $\sigma_{w^{(2)}}^2$) and the observational noise variance parameter $\sigma_{\text{noise}}^2$ are all equal to 1, but the proof generalizes for any positive values of these parameters. 
	
	We abuse notation by letting $w^{(1)}_d \in \mathbb{R}^{K}$ denote the input-layer weight parameters corresponding to \textit{input dimension} $d$ (i.e., and going to all $K$ hidden units),  $w^{(1)}_k \in \mathbb{R}^{D}$ denote the input-layer weight parameters corresponding to \textit{hidden unit} $k$ (i.e., and coming from all $D$ input dimensions), and $w^{(1)}_{kd} \in \mathbb{R}$ denote the single weight parameter corresponding to both input dimension $d$ and hidden unit $k$. We also let $w^{(2)}$ and $b^{(2)}$ denote all $K$ output-layer weight and bias parameters, respectively. 
	
	Recall we define $\theta = \{(w^{(l)}_k, b^{(l)}_k) \}$ as the collection of all model parameters. To avoid complicated subscripts, for any subset $\theta_0 \subset \theta$ of the parameters, define $\phi[\theta_0]$ as the set of variational parameters corresponding to the subset of parameters $\theta_0$. Notice $\phi[\theta_0]$ has twice as many elements as $\theta_0$, since, under the assumed mean-field Gaussian variational distribution, each parameter has a mean and variance variational parameter. So, for example, $\phi[w_d^{(1)}] \in  \mathbb{R}^{2K}$ denotes the input-layer variational parameters corresponding to the $d$th input. Similarly, define $\mu[\theta_0]$ and $\sigma^2[\theta_0]$ as the corresponding variational mean and variance parameters, respectively. We let $\phi$ denote the set of all variational parameters (i.e., $\phi=\phi[\theta]$).
	
	For any $M\in\{1,2,\dotsc\}$ and any subset $\theta_0 \subset \theta$ of the parameters, define $R_M(\phi[\theta_0])$ as the KL divergence of the variational distribution $q_{\phi[\theta_0]}$
	to a $\mathcal{N}(0, \frac{1}{M} I_{|\theta|})$ prior distribution, where $|\theta|$ is the number of elements in $\theta$ and $I_{|\theta|}$ is the $|\theta|\times |\theta|$ identity matrix:
	\begin{align}
		R_M(\phi[\theta_0]) &:= \KL*
		{q_{\phi[\theta_0]}}
		{\mathcal{N}\left(0,\frac{1}{M} I_{|\theta|} \right)}
		\\
		&=  \KL*
		{\mathcal{N}\left(\mu[\theta_0], \text{diag}(\sigma^2[\theta_0]) \right)}
		{\mathcal{N}\left(0,\frac{1}{M} I_{|\theta|} \right) }
		\\
		&= \frac{1}{2} \sum_{i=1}^{|\theta|} \left[M \mu[\theta_{0,i}]^2 + M\sigma^ 2[\theta_{0,i}] -1 - \log M \sigma^2[\theta_{0,i}] \right],
		\label{eq:theorem_kl}
	\end{align}
	where $\theta_{0,i}$ denotes the $i$th element of $\theta_0$.
	
	With this notation, up to an additive constant the negative ELBO for a mean-field variational BNN of width $K$ can be written as:
	\begin{align}
		\text{Loss}(\phi )
		&:=-\text{ELBO}(\phi)
		\\
		&= -\mathbb{E}_{\theta\sim q_\phi}[\log \mathcal{L}(\theta)] + \mathbb{K}\mathbb{L}[q_\phi || p(\theta)]
		\\ 
		&-\underbrace{\frac{1}{2}\sum_{n=1}^N \mathbb{E}_{\theta\sim q_\phi}\left(y_n - f(x_n, \theta) \right)^2}
		_{:=\text{Error}(\phi)} 
		+ \underbrace{\sum_{d=1}^D R_1\left(\phi[w^{(1)}_d] \right) + R_1\left(\phi[b^{(1)}] \right) + R_K\left(\phi[w^{(2)}] \right)}
		_{:=\text{Reg}(\phi)},
	\end{align}
	where $\text{Error}(\phi)$ and $\text{Reg}(\phi)$, respectively, describe the contribution of fitting the data and the KL regularization to the loss. 

	For any width $K\in \{1, 2, \dotsc \}$, let $\phiprior$ be the variational parameters such that the variational distribution $q_{\phiprior}$ is equal to the prior distribution. In other words, the variational parameters where the mean parameters are zero, the input-layer variance parameters are 1, and the output-layer variance parameters are $1/K$. Since $q_{\phiprior}$  is the prior distribution, $\text{Reg}(\phiprior)=0$.

\paragraph{Step 1: evaluate the ELBO at the prior}

We will show that the loss evaluated at the prior, $\text{Loss}(\phiprior)$, does not depend on $K$. This will provide a lower bound on the loss evaluated at the optimal parameters, $\text{Loss}(\phipost)$, that holds for any $K$, which will enable showing the variational distribution needs to stick near the prior. 

To see this, first consider the first two moments of the prior predictive, which are easy to compute because the parameters are independent under the prior:
\begin{align}
	\mathbb{E}_{\phiprior} \left[ f(x_n, \theta)\right] 
	&= \mathbb{E}\left[ \sum_{k=1}^K w^{(2)}_k \psi(w^{(1)T}_k x_n + b^{(1)}_k )  \right] 
	\\
	&=  \sum_{k=1}^K \cancelto{0}{\mathbb{E}\left[w^{(2)}_k\right]} \mathbb{E}\left[\psi(w^{(1)T}_k x_n + b^{(1)}_k ) \right]
	\\
	&= 0
\end{align}
and
\begin{align}
	\mathbb{E}_{\phiprior} \left[ f(x_n, \theta)^2\right]
	&=  \mathbb{V}[f(x_n, \theta)] + \cancelto{0}{\mathbb{E}[f(x_n, \theta)]^2}
	\\
	&= \mathbb{V}\left[ \sum_{k=1}^K w^{(2)}_k \psi(w^{(1)T}_k x_n + b^{(1)}_k ) \right] 
	\\
	&= \sum_{k=1}^K \underbrace{\mathbb{V}\left[w^{(2)}_k\right]}_{1/K} \underbrace{\mathbb{V}\left[\psi(w^{(1)T}_k x_n + b^{(1)}_k ) \right]}_{:= V(x_n)}
	\label{eq:theorem_Vx}
	\\
	&= K \frac{1}{K} V(x_n)
	\\
	&= V(x_n), 
\end{align}
	where we define $V(x_n)$ in Equation \ref{eq:theorem_Vx}. Since the input-layer parameters of each hidden unit, $w^{(1)}_k$ and $b^{(1)}_k$, have the same distribution under the prior for all hidden units, $V(x_n)$ is the same for all  hidden units. Thus we can pull $V(x_n)$ outside of the sum over $k$. With these two moments computed, we can compute the loss under the prior, which will not depend on $K$.
	\vspace{.5cm}
\begin{align}
\text{Loss}(\phiprior) &= \text{Error}(\phiprior) + \cancelto{0}{\text{Reg}(\phiprior)}
\\
&= \frac{1}{2}\sum_{n=1}^N \mathbb{E}_{\phiprior}\left[ \left(y_n - f(x_n, \theta) \right)^2 \right]
\\
&= \frac{1}{2}\sum_{n=1}^N \mathbb{E} \left[y_n^2 - 2y_n f(x_n, \theta) + f(x_n, \theta)^2 \right]
\\
&= \frac{1}{2}\sum_{n=1}^N \left(y_n^2 - 2y_n \cancelto{0}{\mathbb{E}\left[f(x_n, \theta)\right]} + \mathbb{E} \left[ f(x_n, \theta)^2\right] \right)
\\
&=\frac{1}{2}\sum_{n=1}^N\left(y_n^2 + V(x_n) \right)
\\
&:= C_X,
\end{align}
where $X$ is the collection of all $N$ training observations. 
Notice $C_X$ does not depend on $K$.

\paragraph{Step 2: bound the KL regularization}

For any width $K=1,2,\dotsc$, let $[\hat{\mu}_K, \hat{\sigma}^{2}_K] = \phipost  \in \text{argmin}_\phi\  \text{Loss}(\phi)$ be the variational parameters that minimize the loss. Then the minimum of the loss is bounded above by the loss evaluated at $\phiprior$, which we showed does not depend on $K$. In other words,
\begin{align}
	\text{Loss}(\phipost) 
	&= \text{Error}(\phipost) + \text{Reg}(\phipost)
	\\
	&\le \text{Loss}(\phiprior)
	\label{eq:theorem_because_optimal}
	\\
	&= C_X
\end{align}
To provide further explanation, Equation \ref{eq:theorem_because_optimal} holds because otherwise $\phipost$ would not be optimal. In other words, a loss of $C_X =\text{Loss}(\phiprior)$ could always be achieved for any $K$ by setting $q_{\phipost}$ to the prior (i.e., by setting $\phipost=\phiprior$), so the optimal $\phipost$ cannot achieve a worse loss. 

Because $C_X$ does not depend on $K$, it follows that for any $K\in\{1,2,\dotsc\}$,
\begin{align}
	\text{Error}(\phipost) + \text{Reg}(\phipost) &\le C_X
	\\
	\text{Reg}(\phipost) &\le C_X - \text{Error}(\phipost) 
	\\
	\implies \text{Reg}(\phipost) &\le C_X
	\\
	\sum_{d=1}^D R_1\left(\phipost[w^{(1)}_d] \right) + R_1\left(\phipost[b^{(1)}] \right) + R_K\left(\phipost[w^{(2)}] \right) & \le C_X
\end{align}
Therefore, the regularization of the optimal variational parameters $\text{Reg}(\phipost)$ is bounded above by $C_X$. Furthermore, since each of the regularization terms is non-negative, each is less than the bound:
\begin{align}
	R_1\left(\phipost[w^{(1)}_d] \right) &\le C_X, \quad \forall d\in \{1, \dotsc, D\}
	\\
	R_1\left(\phipost[b^{(1)}] \right) &\le C_X
	\\
	R_K\left(\phipost[w^{(2)}] \right) &\le C_X
\end{align}
Additionally, since the contribution of the variance parameters to the KL divergence in Equation \ref{eq:theorem_kl} is non-negative (i.e. since the function $g(\sigma^2) := M\sigma^2 - 1 - \log M\sigma^2 \ge 0$ for all $M>0$), the squared mean parameters, summed over all hidden units, are also bounded (i.e., since $a + b \le c \implies a \le c$ if $b\ge0$):
\begin{align}
	\frac{1}{2}\sum_{k=1}^K \left(\hat{\mu}_K[w^{(1)}_{kd}] \right)^2  &\le C_X, \quad \forall d\in \{1, \dotsc, D\}
	\label{eq:theorem_w1_bound}
	\\
	\frac{1}{2}\sum_{k=1}^K \left( \hat{\mu}_K[b^{(1)}_{k}] \right)^2  &\le C_X
	\label{eq:theorem_b1_bound}
	\\
	\frac{1}{2}\sum_{k=1}^K K \left(\hat{\mu}_K[w^{(2)}_{k}] \right)^2  &\le C_X
	\label{eq:theorem_w2_bound}
\end{align}
	
\paragraph{Step 3: bound the posterior mean}

Using the bounds on the optimal variational mean parameters in Equations \ref{eq:theorem_w1_bound}, \ref{eq:theorem_b1_bound}, and \ref{eq:theorem_w2_bound}, we show the absolute value of the posterior mean converges to zero. 
\begin{align}
	\left | \mathbb{E}_{\theta \sim q_{\phipost}}[f(x^*, \theta)]  \right | 
	&= \left| \mathbb{E}\left[ \sum_{k=1}^K w^{(2)}_k \psi(w^{(1)T}_k x^* + b^{(1)}_k )  \right] \right|
	\\
	&= \left| \sum_{k=1}^K \mathbb{E} \left[w^{(2)}_k \right] \mathbb{E}\left[ \psi(w^{(1)T}_k x^* + b^{(1)}_k ) \right] \right|
	\label{eq:theorem_use_mfvi}
	\\
	&\le \left( \sum_{k=1}^K \mathbb{E} \left[w^{(2)}_k \right]^2 \right)^{1/2} \left( \sum_{k=1}^K \mathbb{E}\left[ \psi(w^{(1)T}_k x^* + b^{(1)}_k ) \right]^2 \right)^{1/2}
	\label{eq:theorem_cauchy}
	\\
	&\le \left( \sum_{k=1}^K \left(\hat{\mu}_K[w^{(2)}_{k}] \right)^2 \right)^{1/2} \left( \sum_{k=1}^K \mathbb{E}\left[ \psi(w^{(1)T}_k x^* + b^{(1)}_k ) \right]^2 \right)^{1/2}
	\\
	&\le \left( \frac{2 C_X}{K} \right)^{1/2} \left( \sum_{k=1}^K \mathbb{E}\left[ \psi(w^{(1)T}_k x^* + b^{(1)}_k ) \right]^2 \right)^{1/2} 
	\label{eq:theorem_almost_there}
\end{align}
where Equation \ref{eq:theorem_use_mfvi} follows because we assume a mean-field posterior and Equation \ref{eq:theorem_cauchy} follows from the Cauchy-Schwarz inequality and the last equation follows from Equation \ref{eq:theorem_w2_bound}.

To bound the second term in Equation \ref{eq:theorem_almost_there}, consider the distribution of the pre-activations $z_k := w^{(1)T}_k x^* + b^{(1)}_k$. Define
\begin{align}
	\hat{\mu}_K[z_k]&:= \sum_{d=1}^D \hat{\mu}_K[w^{(1)}_{kd}] x^*_d + \hat{\mu}_K[b^{(1)}_{k}]
	\\
	\hat{\sigma}^2_K[z_k] &:= \sum_{d=1}^D  \hat{\sigma}^2_K[w^{(1)}_{kd}] x_d^{*2} + \hat{\sigma}^2_K[b^{(1)}_{k}].
\end{align}
Then, since each of the parameters is Gaussian distributed and independent under the mean-field variational posterior distribution, $z_k \sim \mathcal{N}(\hat{\mu}_K[z_k], \hat{\sigma}^2_K[z_k])$.

Next, we use Lemma \ref{lemma4} to bound the sum of the squared means of $z_k$'s so that we can later apply Lemma \ref{lemma3}. For $d=1,\dotsc,D$, define $a_{kd} = \hat{\mu}_K[w^{(1)}_{kd}] x^*_d$ and for $d=D+1$ define $a_{kd} = \hat{\mu}_K[b^{(1)}_{k}]$, so that $\hat{\mu}_K[z_k] = \sum_{d=1}^{D+1} a_{kd}$. Notice for any $d=1,\dotsc,D$, 
\begin{align}
	\sum_{k=1}^K a_{kd}^2 = x_d^{*2} \sum_{k=1}^K \left( \hat{\mu}_K[w^{(1)}_{kd}] \right)^2  \le 2 x_d^{*2} C_X
\end{align}
by Equation \ref{eq:theorem_w1_bound} and for $d=D+1$:
\begin{align}
	\sum_{k=1}^K a_{kd}^2 = \sum_{k=1}^K \left(\hat{\mu}_K[b^{(1)}_{k}] \right)^2  \le 2 C_X
\end{align}
by Equation \ref{eq:theorem_b1_bound}.
Therefore, by Lemma \ref{lemma4}:
\begin{align}
	\sum_{k=1}^K \left(\hat{\mu}_K[z_k]\right)^2 =
	\sum_{k=1}^K \left( \sum_{d=1}^{D+1} a_{kd} \right)^2 
	&\le (D+1) \left( \sum_{d=1}^{D} 2 x_d^{*2} C_X + 2C_X\right) 
	\\
	&= 2(D+1)C_X \left( \sum_{d=1}^{D} x_d^{*2}  + 1\right) 
	\\
	&:= C_{X, x^*}.
	\label{eq:theorem_C_Xx}
\end{align}

We can now put all the results together to bound the second term in Equation \ref{eq:theorem_almost_there}. 
\begin{align}
	\sum_{k=1}^K \mathbb{E}_{\phipost}\left[ \psi(w^{(1)T}_k x + b^{(1)}_k ) \right]^2
	&= \sum_{k=1}^K \mathbb{E}\left[ \psi(z_k) \right]^2
	\\
	&= \sum_{k=1}^K \mathrm{erf}\left(\frac{\hat{\mu}_K[z_k]}{\sqrt{1+2\hat{\sigma}^2_K[z_k]}} \right)^2
	\label{eq:theorem_lemma1}
	\\
	&\le \sum_{k=1}^K \left(1-\exp\left(-\frac{4}{\pi} \frac{\hat{\mu}_K[z_k]^2}{1+2\hat{\sigma}^2_K[z_k]} \right) \right)
	\label{eq:theorem_lemma2}
	\\
	&\le K - \sum_{k=1}^K \exp\left(-\frac{4}{\pi} \frac{\hat{\mu}_K[z_k]^2}{1+2\hat{\sigma}^2_K[z_k]} \right)
	\\
	&\le K - \sum_{k=1}^K  \exp\left(-\frac{4}{\pi} \hat{\mu}_K[z_k]^2\right)
	\label{eq:theorem_drop_s2}
	\\
	&\le K - K \exp\left(-\frac{1}{K} \frac{4}{\pi} C_{X, x^*} \right)
	\label{eq:theorem_lemma3}
	\\
	&= 	K\left(1 - \exp\left(-\frac{1}{K} \frac{4}{\pi} C_{X, x^*} \right) \right),
\end{align}
where Equation \ref{eq:theorem_lemma1} follows from Lemma \ref{lemma1}, Equation \ref{eq:theorem_lemma2} follows from Lemma \ref{lemma2}, and Equation \ref{eq:theorem_lemma3} follows from Lemma \ref{lemma3} (using the bound in Equation \ref{eq:theorem_C_Xx}). Equation \ref{eq:theorem_drop_s2} follows because $\forall a,b\in\mathbb{R}$, one can show $\exp(-a^2/(1+b^2)) \ge \exp(-a^2)$. 

Now plug into Equation \ref{eq:theorem_almost_there}:
\begin{align}
	\left | \mathbb{E}_{\theta \sim q_{\phipost}}[f(x^*, \theta)]  \right | 
	&\le \left( \frac{2 C_X}{\cancel{K}} \right)^{1/2} \left(\cancel{K}\left(1 - \exp\left(-\frac{1}{K} \frac{4}{\pi} C_{X, x^*} \right) \right) \right)^{1/2}
	\\
	& = \left( 2 C_X \right)^{1/2} \left(1 - \exp\left(-\frac{1}{K} \frac{4}{\pi} C_{X, x^*} \right) \right)^{1/2}
	\\
	&\to 0 \text{ as } K\to\infty.
\end{align}
Lastly, since the absolute value of the posterior mean converges to zero, so does posterior mean without the absolute value. 

\end{proof}

\subsection{Additional experiments}

First we illustrate that regardless of width, the BNN priors in function space are approximately the same and resembles that of the NNGP. Figure \ref{fig:more_priors} shows the prior of BNNs of increasing width (rows) as compared to the NNGP prior (bottom row) based on 1,000 function samples from the prior. The prior settings are the same as in the main paper. The horizontal axis is the 1-dimensional input, $x$. As expected, all of the prior predictive distributions (left column) have the same mean and variance. However, this is only the first two moments of the distribution. To measure the smoothness of the function-space prior, in the right column we show a histogram of the $x$ locations at which a function sample crosses the line $y=0$, called an \textit{upcrossing}. The number of upcrossings is a measure of the inverse lengthscale (i.e., inverse smoothness). All models exhibit approximately the same number and distribution of upcrossings. 

\begin{figure}[H]
	\centering
	\includegraphics[width=1\textwidth]{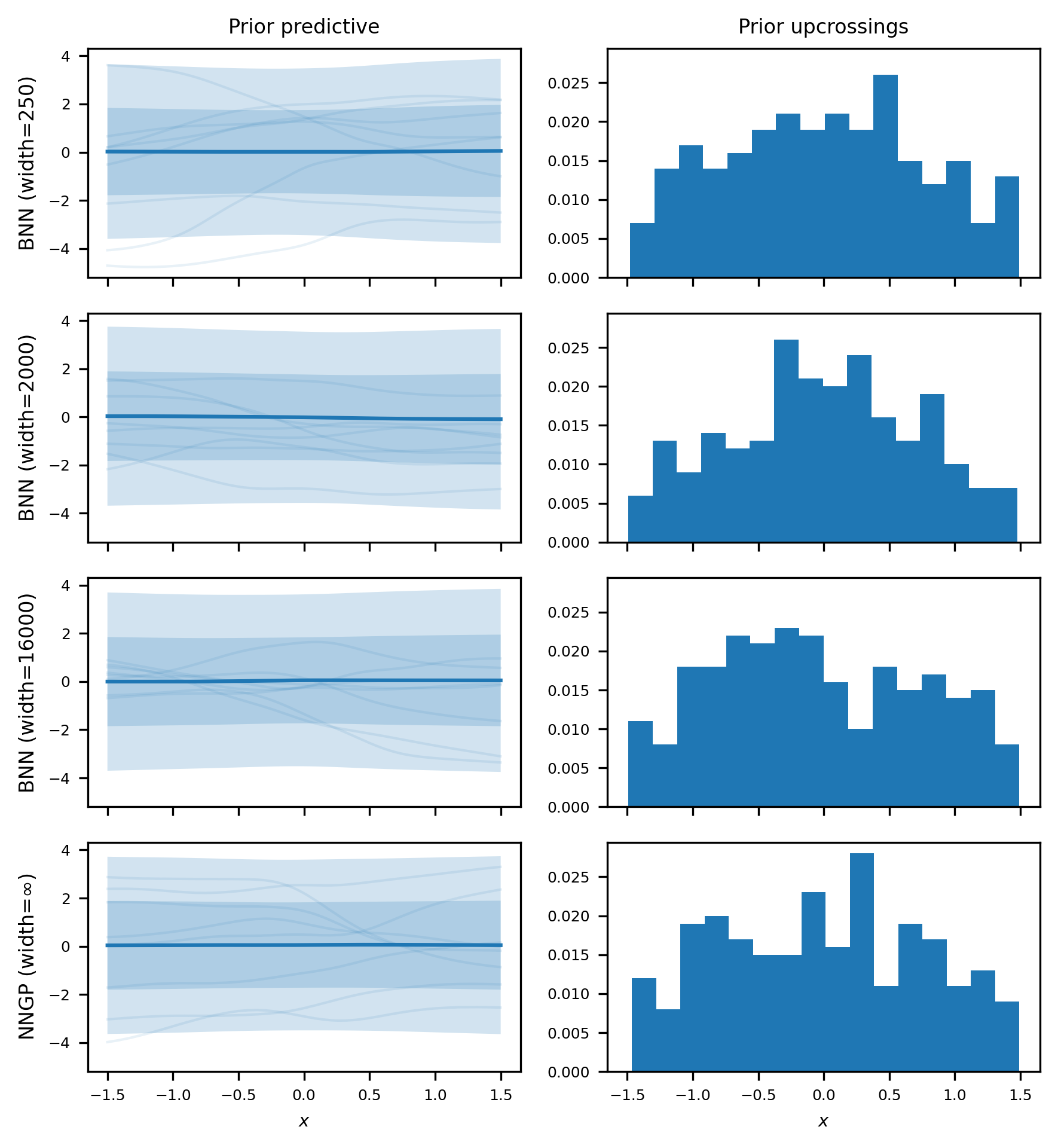}
	\vspace{-.85cm}
	\caption{Prior predictive distributions and histogram of upcrossings of $y=0$ based on 1000 prior function samples.The shaded regions constitute $\pm 1$ (darker shade) and $\pm 2$ (lighter shade) standard deviations around the means (solid lines), with a few samples drawn faintly.}
	\label{fig:more_priors}
\end{figure}

Next, using the priors in \ref{fig:more_priors} we infer the mean-field variational posteriors based on the three datasets (Figure \ref{fig:more_posteriors}). As expected from \ref{th:1}, as the width increases (rows) the mean-field variational posteriors  begin to ignore the data as they converge to zero. On the other hand, the NNGP (bottom row) fits the data nicely. Recall that while the true BNN posterior converges to the NNGP posterior as the width $K$ tends to infinity, the mean-field variational posterior does not.  

\begin{figure}[H]
	\centering
	\includegraphics[width=1\textwidth]{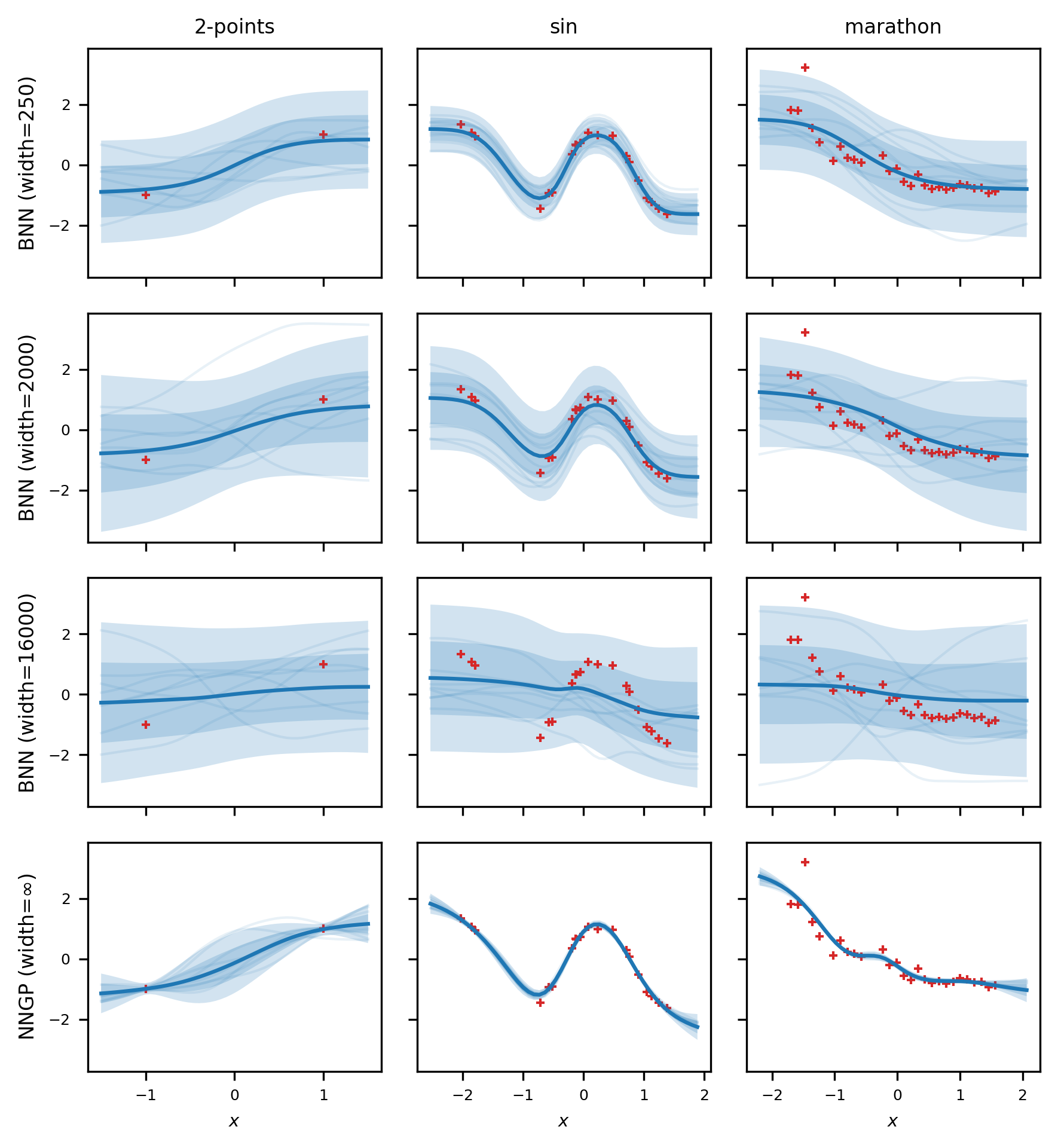}
	\vspace{-.85cm}
	\caption{Posterior predictive distributions for mean-field variational BNNs of different widths and trained on different datasets. For a large width, the mean-field variational BNN ignores the data. The shaded regions constitute $\pm 1$ (darker shade) and $\pm 2$ (lighter shade) standard deviations around the means (solid lines). All estimates are based on 1000 function samples (a few of which are drawn faintly). }
	\label{fig:more_posteriors}
\end{figure}

Figure \ref{fig:varational_params} shows the distribution of the variational parameters after optimizing the ELBO. We use the ``2 points'' dataset. To be clear, this is a kernel density estimate of the trained variational parameters themselves across hidden units of the network, not the variational distributions that the variational parameters define. Going from left-to-right, as the network width $K$ increases, the variational parameters move closer to the $\mathcal{N}(0,1)$ prior. Note that the prior variance of the output-layer weights is still effectively scaled by $1/K$ but in our implementation the scaling is performed in the function evaluation, which enables all parameters to have a $\mathcal{N}(0,1)$ prior. 

\begin{figure}[H]
	\centering
	\begin{subfigure}[b]{0.33\textwidth}
		\centering
		\includegraphics[width=1\textwidth]{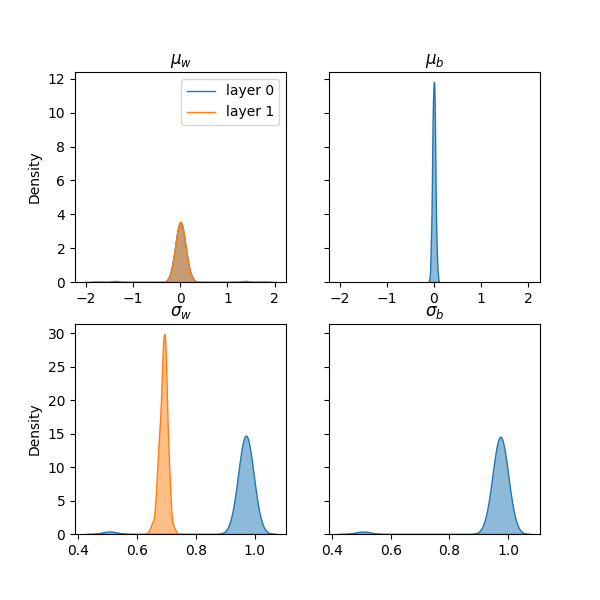}
		\caption{$\text{width} = 250$}
	\end{subfigure}
	\begin{subfigure}[b]{0.33\textwidth}
		\centering
		\includegraphics[width=1\textwidth]{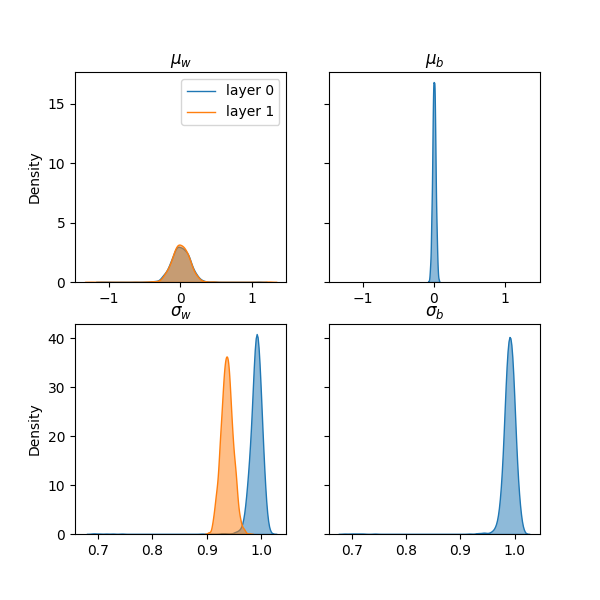}
		\caption{$\text{width} = 2000$}
	\end{subfigure}
	\begin{subfigure}[b]{0.33\textwidth}
		\centering
		\includegraphics[width=1\textwidth]{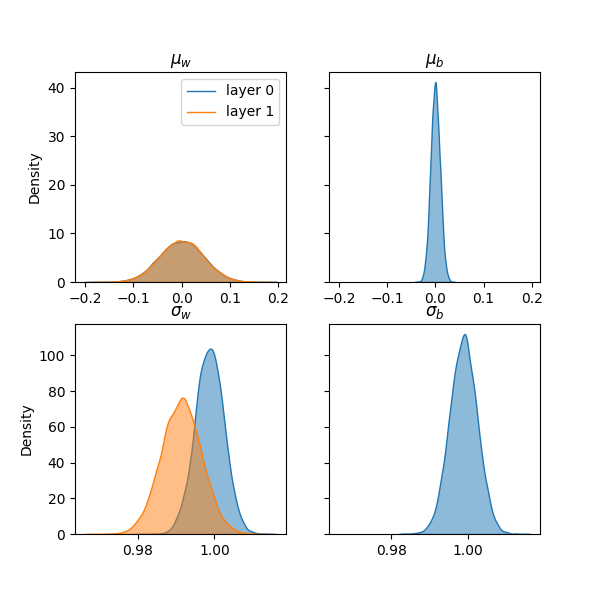}
		\caption{$\text{width} = 16,000$}
\end{subfigure}
	\caption{Kernel density estimates of the distribution of trained variational parameters using the ``2-points'' dataset. We break it down by layer, parameter type (weight or bias), and variational parameter type (mean or variance). In our implementation, since we scale by the prior parameters in the function evaluation (i.e., we use a ``neural tangent kernel'' scaling) all variational parameters have a $\mathcal{N}(0,1)$ prior distribution. Notice the trained variational variance parameters (bottom panels) shift closer the prior value of 1 as the network size increases. For the largest network, all variational parameters are near their prior values.}
	\label{fig:varational_params}
\end{figure}

In all variational inference experiments, we use 20,000 epochs of full-batch gradient descent with a learning rate of 0.001 and a momentum of 0.9 optimization. We use gradient clipping and cosine annealing of the learning rate, with warm restarts every 500 epochs \cite{loshchilov_2017}. To evaluate the ELBO, we use the analytical form of the KL divergence and the reparameterization trick \cite{kingma_2014} with 64 samples to approximate the expected log  likelihood term. 

\end{document}